\documentclass{article}
\usepackage{spconf,amsmath,graphicx}
\usepackage{cite}
\usepackage{amsmath,amssymb,amsfonts}
\usepackage{algorithm,algorithmic}
\usepackage{graphicx}
\usepackage{textcomp}
\usepackage{xcolor}
\usepackage{amsthm}
\usepackage[numbers]{natbib}
\usepackage[subtle]{savetrees}
\theoremstyle{definition}
\newtheorem{definition}{Definition}
\newtheorem{example}{Example}
\newtheorem{theorem}{Theorem}
\newtheorem{corollary}{Corollary}
\newtheorem{assumption}{Assumption}
\newenvironment*{proofsk*}{\paragraph{Proof Sketch:}}{\hfill$\square$}

\title{On the Value of Stochastic Side Information in Online Learning}
\name{
Junzhang Jia, Xuetong Wu, Jamie Evans, Jingge Zhu
}
\address{University of Melbourne\\
    Department of Electrical and Electronic Engineering\\
    Parkville, Victoria, Australia}

\begin{document}
%
\maketitle
\begin{abstract}
We study the effectiveness of stochastic side information in deterministic online learning scenarios. We propose a forecaster to predict a deterministic sequence where its performance is evaluated against an expert class. We assume that certain stochastic side information is available to the forecaster but not the experts. We define the minimax expected regret for evaluating the forecaster's performance, for which we obtain both upper and lower bounds. Consequently, our results characterize the improvement in the regret due to the stochastic side information. Compared with the classical online learning problem with regret scales with $O(\sqrt{n})$, the regret can be negative when the stochastic side information is more powerful than the experts. To illustrate, we apply the proposed bounds to two concrete examples of different types of side information.
\end{abstract}
\section{Introduction} \label{sec:intro}
The online learning problem aims to make predictions for probabilistic/deterministic instances which arrive sequentially, and has become significantly popular in game theory and learning theory fields recently. \citet{merhav1998universal} studied online learning problems for stochastic setup from an information-theoretic perspective, followed by \cite{wu2021online}. 
In the deterministic setting, we will usually introduce a class of competitive predictors providing advice to the forecaster, namely the expert class \cite{vovk1998game,merhav1998universal}, and the learning performance is evaluated by the regret, i.e. the loss gap between the proposed forecaster and the best expert. To effectively leverage the experts, \citet{littlestone1994weighted} proposed a weighted majority algorithm, and the follow-up works such as \cite{vovk1990aggregating,vovk1998game,auer2002nonstochastic} further proposed the randomized algorithms which produce logarithmic regret. In a more specific setup, \citet{haussler1995tight} considered binary and continuous instance spaces and provided an $\Omega(\sqrt{n\log N})$ worst-case regret, where $n$ is the sample size and $N$ is the number of experts. With respect to different loss functions, \citet{cesa2006prediction} and \citet{vanli2014unified} provided explicit upper and lower bounds on the regret for absolute loss and squared loss, respectively. 

As a common situation in practice, the forecaster could access some additional resources which we call it \textit{side information}, that may provide some useful knowledge on the sequence of interest. \citet{cover1996universal} first studied a portfolio investment problem where the sequence of interest is the stock vectors that may depend on some finite-valued states (as side information), and their proposed forecaster can achieve the same wealth as the best side information dependent investment strategy. \citet{xie2000asymptotic} studied the case when the sequence of interest is generated according to a pair-wise parametric distribution conditioning on the side information, and derived an logarithmic upper bound of the minimax regret. \citet{cesa2006prediction} analyzed the problem with an additional (deterministic) side sequence, then the learning performance depends on the occurrences of its agreed symbols compared to the sequence of interest. Recently, \citet{bhatt2021sequential} studied the probabilistic online learning problem where the side information is the auxiliary random symbols generated jointly with the data instance to be predicted, and analyzed the minimax regret under the logarithmic loss.


However, to the best of our knowledge there is no prior work discussing the formulation and effects of the stochastic side information under a deterministic online learning scenario. Inspired by the transfer learning problem \cite{torrey2010transfer} where people transfer the knowledge from one domain (source) to the domain of interest (target) with both the source and target data drawn from different but related distributions, we specify the formation of the side information that may depend on the target sequence with some stochasticity. In a similar spirit, we aim to explore the influence of a stochastic \textit{sequential side information} (SSI) for predicting a \textit{target sequence} of interest. In this paper, we propose a novel problem formulation where a forecaster tries to predicts a deterministic sequence with some stochastic side information, which is not known to the expert class. Then we develop an online learning framework with the expert class where we will additionally leverage the side information for prediction to minimise the regret with respect to the best expert. With the proposed algorithm, we provide both the lower and upper bounds on the minimax regret under the absolute loss, where the target sequence is selected adversarially to maximise the regret. From the results, we show that introducing SSI can improve the typical learning rate in  \cite{merhav1998universal,haussler1995tight,cesa2006prediction} if the side information performs better than the best expert. On the other hand, the side information will not hurt our prediction if it fails to provide much useful information. 




\section{Problem Formulation and Main Results} \label{sec:main}

\subsection{Prediction with Experts and Stochastic Side Information} \label{sec:problem}
We consider the online learning problem for a deterministic target sequence with the side information: we aim to design a forecaster that sequentially predict the outcome of an unknown target sequence $\mathcal{T}_n=(X^T_1, X^T_2, ..., X^T_n)$ where each instance $X^T_t$ takes value in a set $\mathcal{X} \subseteq \mathbb{R}$. The prediction of the forecaster at time $t$, denoted by $\tilde{X}_t$, takes value in a space $\mathcal{D}$ which is a convex and nonempty subset of $\mathbb{R}$, and we also assume $\mathcal{X}\subset\mathcal{D}$. We will compare the forecaster with a class of experts. We use $\mathcal{F}^{\theta}_n$ to denote the prediction sequence made by the expert $\theta$, and $f_t^{\theta}$ to denote the prediction at time $t$. Here we denote by $\theta$ the index of an expert, taking value in an index set $\Lambda=\{1,2,\cdots,N\}$ and $N\in\mathbb{R^+}$ is the number of experts in the class. The performance of the predictions is evaluated by a non-negative loss function $\ell:\mathcal{D}\times\mathcal{X}\mapsto\mathbb{R^+}$. We assume that only the forecaster has access to 
the SSI which may provide extra information on target sequence, which is denoted by $\mathcal{S}_n = (X^S_1, X^S_2, ..., X^S_n)$ for each $X^S_t \in \mathcal{X}$. At each time $t$, the forecaster predicts the current target instance $X_{t}^T$ with previous observations $(X_{1}^T,\ldots,X_{t-1}^T)$ up to time $t-1$ and the corresponding SSI $(X_1^{S},\ldots,X_{t}^S)$ up to time $t$. In other words, the prediction $\tilde{X}_t$ can be regarded as a function of both SSI and target sequences $\tilde{X}_t(X^S_1,X^S_2,...,X^S_t, X^T_1,X^T_2,...,X^T_{t-1})$. We also use $\tilde{\mathcal{T}}_n:=(\tilde{X}_1,\tilde{X}_2,\dots,\tilde{X}_n)$ to denote the sequence of the predictions.

Following the common assumption in the literature \cite{cesa2006prediction,cesa2021online} for deterministic online learning problems, we assume that the target sequence is an arbitrary sequence. It can even be viewed as adversarially chosen by the ``environment" with the knowledge of the prediction rule of the forecaster. However, we assume the SSI is generated in a conditional independent stochastic fashion by $P^n(\mathcal{S}_n|\mathcal{T}_n)=\prod_t P(X_t^S|X_t^T)$ with some (known) conditional probability distribution $P(X|Y)$.

To evaluate the performance of the prediction sequences, we firstly define the cumulative loss $L$ which takes two sequences $A_n:=(a_1,a_2,...,a_n)$ and $B_n:=(b_1,b_2,...,b_n)$ with length $n$ as:
\begin{align}
    L(A_n,B_n) = \sum\limits_{t=1}^{n} \ell(a_t,b_t)
\end{align}
We use the absolute loss $\ell(a,b)=|a-b|$ throughout this paper in order to derive the lower bounds \cite{cesa2006prediction}. 
We then define the \textit{regret} for the deterministic online learning problems as the difference between the cumulative loss between our forecaster and the best expert:
\begin{equation}\label{eq:Regret}
L(\tilde{\mathcal{T}}_n,\mathcal{T}_n)-\min_{\theta}L(\mathcal{F}^{\theta}_n,\mathcal{T}_n).
\end{equation}

\subsection{Minimax Expected Regret} \label{sec:algorithm}
In this section, we consider a problem of minimising the expected regret for a worst-case target sequence $\mathcal{T}_n$, i.e. $\mathcal{T}_n$ that maximises the expectation (w.r.t the side sequence) of (\ref{eq:Regret}). To this end, we will study the \textit{minimax expected regret} defined as follows. Unless specified, the expectation is always taken over the SSI conditional distribution $P^n(\mathcal{S}_n|\mathcal{T}_n)$.
\begin{small}
\begin{equation} \label{eq:Rminimax}
R(n):= \inf\limits_{\tilde{\mathcal{T}}_n} \sup\limits_{\mathcal{T}_n} \mathop{\mathbb{E}}\limits_{\mathcal{S}_n} \{L(\tilde{\mathcal{T}}_n(\mathcal{S}_n,\mathcal{T}_n),\mathcal{T}_n)-\min\limits_{\theta}L(\mathcal{F}^{\theta}_n,\mathcal{T}_n)\}. 
\end{equation}
\end{small}
To evaluate the usefulness of the SSI, we introduce the maximum likelihood estimation of the target instances $X^T_t$ given $X^S_t$: 
\begin{align}
    \hat{X}^T_t(X^S_t)=\underset{X^T_t}{\arg\max}\ P(X^S_t|X^T_t).
\end{align} 
Then we denote the maximum likelihood prediction sequence by $\hat{\mathcal{T}}_n(\mathcal{S}_n):=(\hat{X}^T_t(X^S_1),\hat{X}^T_t(X^S_2),\dots,\hat{X}^T_t(X^S_n))$. Furthermore, we make the assumption that the expected cumulative loss induced by  $\hat{\mathcal{T}}_n(\mathcal{S}_n)$ is upper bounded in the following.
\begin{assumption}\label{assump:D}
For any target sequence $\mathcal{T}_n$, it holds that
\begin{equation}
\mathop{\mathbb{E}}_{\mathcal{S}_n}[L(\hat{\mathcal{T}}_n(\mathcal{S}_n),\mathcal{T}_n)]\leq C_S(n),
\end{equation}
where $C_S(n)$ is a finite value depending on $n$ for $n<\infty$.
\end{assumption}
\noindent Assumption~\ref{assump:D} does not significantly restrict the SSI as we only require that the total loss induced by $\hat{\mathcal{T}}_n(\mathcal{S}_n)\in\mathcal{X}$ is bounded and we do not specify the function $C_S(n)$ at this stage. Clearly, $C_S(n)$ depends the SSI through the conditional distribution $P(X^S_t|X^T_t)$, for which we will give two concrete examples in Section \ref{sec:example}. 



With definitions in place, we introduce Algorithm~\ref{alg:exp}, which we call \textit{Exp3 with SSI}. This algorithm is an extension of the classical Exponentially Weighted Average (Exp3) algorithm \cite{cesa2006prediction}, which uses an exponentially updated mixture of the experts as the forecaster. Our algorithm further treats the maximum likelihood estimator $\hat{X}^T_t(X^S_t)$ as an additional expert, so that the prediction made by the forecaster will partially depend on the information provided by the SSI.

\begin{algorithm} 
\caption{Exp3 with SSI}
\begin{algorithmic}[1]\label{alg:exp}
 \STATE Initialize the weights for the SSI $w^S_1$ and all experts $w^{\theta}_1$ to be 1;
 \FOR {$t = 1$ to $n$}
 \STATE Receive $X^S_t$;
 \STATE $\tilde{X}_t = \frac{w^S_t\hat{X}^T_t(X^S_t)+\sum_{\theta=1}^{N}w_{t}^{\theta}f_{t}^{\theta}}{w^S_t+\sum_{\theta=1}^{N} w_{t}^{\theta}}$;
 \STATE Receive $X^T_t$;
 \STATE The experts suffer the loss $\ell(f^{\theta}_t,X^T_t)$;
 \STATE The SSI suffers the loss $\ell(\hat{X}^T_t(X^S_t),X^T_t)$;
 \STATE $w^S_{t+1} = w^S_te^{-\eta \ell(\hat{\mathcal{T}}_n(\mathcal{S}_n),X^T_t)}$;
 \FOR {$\theta = 1$ to $N$}
 \STATE $w^{\theta}_{t+1} = w^{\theta}_te^{-\eta \ell(f^{\theta}_t,X^T_t)}$
 \ENDFOR
 \ENDFOR
\RETURN $\tilde{\mathcal{T}}_n=(\tilde{X}_1,\tilde{X}_2,\dots,\tilde{X}_n)$
\end{algorithmic} 
\end{algorithm}

Since the decision space $\mathcal{D}$ is convex and nonempty, the prediction $\tilde{X}_t$ in Algorithm \ref{alg:exp} formed by a linear combination of the expert $f^{\theta}_t$ and $\hat{X}^T_t(X^S_t)$ is also guaranteed to lie in the decision space $\mathcal{D}$. In the following theorem, we give an upper bound on the minimax regret with the proposed algorithm.
\begin{theorem}\label{th:upper}
With Assumption~\ref{assump:D}, the minimax expected regret in (\ref{eq:Rminimax}) is upper bounded by:
\begin{align}
R(n) \leq \sqrt{\frac{n}{2}\log{(N+1)}} + \min\{C_S(n)-L^*(n),0\},
\end{align}
where we define $L^*(n)=\inf\limits_{\mathcal{T}_n}\min\limits_{\theta}L(\mathcal{F}^{\theta}_n,\mathcal{T}_n)$.
\end{theorem}
\begin{proof}
To simplify the notations, we denote by $L_{\theta}$ the loss induced by the loss induced from the expert ${\theta}$, i.e. $L_{\theta}=L(\mathcal{F}^{\theta}_n,\mathcal{T}_n)$.
We also denote $L_{min}$ as the minimum cumulative loss among all the experts and the SSI:
\begin{equation}
L_{min}=\min\{L(\hat{\mathcal{T}}_n(\mathcal{S}_n),\mathcal{T}_n) , \min\limits_{\theta}L_{\theta}\}.
\end{equation}

Then by adding and subtracting the term $L_{min}$ in $R$, we have
\begin{align}
R(n) =& \inf\limits_{\Xi_n} \sup\limits_{\mathcal{T}_n} \mathop{\mathbb{E}}\limits_{\mathcal{S}_n} \{L(\Xi_n(\mathcal{S}_n,\mathcal{T}_n),\mathcal{T}_n)-L_{min} \nonumber \\
&+L_{min}-\min\limits_{\theta}L_{\theta}\}\\
\leq&\inf\limits_{\Xi_n} \sup\limits_{\mathcal{T}_n} \mathop{\mathbb{E}}\limits_{\mathcal{S}_n} \{L(\Xi_n(\mathcal{S}_n,\mathcal{T}_n),\mathcal{T}_n)-L_{min}\} \nonumber \\
&+\sup\limits_{\mathcal{T}_n} \mathop{\mathbb{E}}\limits_{\mathcal{S}_n}\{L_{min}-\min\limits_{\theta}L_{\theta}\}\\
\leq&\sup\limits_{\mathcal{T}_n} \mathop{\mathbb{E}}\limits_{\mathcal{S}_n} \underbrace{\{L(\Xi_n(\mathcal{S}_n,\mathcal{T}_n),\mathcal{T}_n))-L_{min}\} \nonumber}_{R_a}\\
&+\underbrace{\sup\limits_{\mathcal{T}_n} \mathop{\mathbb{E}}\limits_{\mathcal{S}_n}\{L_{min}-\min\limits_{\theta}L_{\theta}\}}_{R_b}\label{eq:upperab}
\end{align}
Then we separately upper bound the quantity $R_a$ and $R_b$. 
We regard the one realization of $X^S$ as an expert, then by the Theorem 2.2 in \cite{cesa2006prediction}, we have
\begin{align}
R_a \leq \frac{\ln{N+1}}{\eta}+\frac{n\eta}{2}=\sqrt{\frac{n}{2}\log(N+1)}\label{eq:ra}
\end{align}
with the optimal selection of the learning factor that $\eta=\sqrt{\frac{8\ln{(N+1)}}{n}}$. Then we can remove the first supremum and expectation in (\ref{eq:upperab}) as $R_a$ is upper bounded by a quantity only depends on $N$ and $n$. Then for $R_b$, we have,
\begin{align}
R_b\leq &\sup\limits_{\mathcal{T}_n} \mathop{\mathbb{E}}\limits_{\mathcal{S}_n}\{L_{min}-\min\limits_{\theta}L_{\theta}\}\\
=&\sup\limits_{\mathcal{T}_n} \mathop{\mathbb{E}}\limits_{\mathcal{S}_n}\{\min\{L(\hat{\mathcal{T}}_n(\mathcal{S}_n),\mathcal{T}_n),\min\limits_{\theta}L_{\theta}\}-\min\limits_{\theta}L_{\theta}\}\\
\overset{(a)}{\leq}&\sup\limits_{\mathcal{T}_n} \{\min\{\mathop{\mathbb{E}}\limits_{\mathcal{S}_n}(L(\hat{\mathcal{T}}_n(\mathcal{S}_n),\mathcal{T}_n)),\min\limits_{\theta}L_{\theta}\}-\min\limits_{\theta}L_{\theta}\}\\
=&\sup\limits_{\mathcal{T}_n} \{\min\{\mathop{\mathbb{E}}\limits_{\mathcal{S}_n}(L(\hat{\mathcal{T}}_n(\mathcal{S}_n),\mathcal{T}_n)-\min\limits_{\theta}L_{\theta}),0\}\}\label{eq40}\\
\leq&\min\{C_S(n)-L^*(n),0\}. \label{eq:rb}
\end{align}
The second term in equation~(\ref{eq40}) is straightforwardly taken from assumption \ref{assump:D}. Inequality (a) holds according to the fact that $L_{\theta}$ does not depend on $\mathcal{S}_n$ and the inequality $\mathbb{E}(\min\{a,b\})\leq\min\{\mathbb{E}(a),\mathbb{E}(b)\}$. The last inequality results from the Assumption \ref{assump:D} and taking the supremum separately. We then complete the proof by adding up eq~(\ref{eq:ra}) and eq~(\ref{eq:rb}).
\end{proof}

The learning rate in its current form is not determined since the rate of $C_S(n)$ and $L^*(n)$ may vary across different cases. In Section \ref{sec:example}, we will provide two specific structures for $C_S(n)$ and $L^*(n)$ with two examples. Notably, Theorem~\ref{th:upper} indicates that the effect of the SSI will be determined by the difference between $C_S(n)$ and $L^*(n)$. In particular, if the expected cumulative loss of the side information is smaller than the loss induced by the best expert, the SSI is indeed helpful for predicting the target sequence. In contrast, when the SSI induces a higher loss compared to the best expert, the regret is upper bounded by $\sqrt{\frac{n}{2}\log (N+1)}$, which is essentially the same as the learning bound without SSI \cite{cesa2006prediction,haussler1995tight,vanli2014unified} when N is large. As a result, the second term in theorem \ref{th:upper} indicates how much the SSI can improve on the regret. To examine the tightness of the proposed upper bound, we also develop a lower bound for a particular outcome space and a decision space in the following theorem. 

\begin{theorem} \label{th:lower}
Consider $\mathcal{X}=\{0,1\}$ and $\mathcal{D}=[0,1]$, with the absolute loss $\ell(x, y)=|x-y|$, we have
\begin{align}
R(n) &\geq \sqrt{\frac{n}{2}\log{(N+1)}}+\left(\xi^* - \frac{1}{2}\right)n,\label{eq:lower}
\end{align}
where $\xi^*=\inf_{\tilde{X}_1}\mathbb{E}_{Z,X^S_1}|\tilde{X}_1(X^S_1) - Z|$, in which $Z$ and $X_1^S$ are jointly distributed according to $P(Z,X_1^S)$, and $Z$ is marginally Bernoulli distributed as $Z\sim \textup{Ber}(\frac{1}{2})$, $X^S_1$ is generated according to the conditional distribution $P(X^S_1|Z)$.
\end{theorem}

\begin{proof}
The proof is different compared with the previous work \cite{cesa2006prediction,haussler1995tight} that the prediction $p_t$ now depends on the target outcome $X^T_t$ by referencing the advice from both experts and the SSI $X^S_t$. First of all, we lower bound the first term in (\ref{eq:Rminimax}) as follows.
\begin{align}
&\inf\limits_{\Xi_n} \sup\limits_{\mathcal{T}_n} \mathop{\mathbb{E}}\limits_{\mathcal{S}_n} \{L(\Xi_n(\mathcal{S}_n,\mathcal{T}_n),\mathcal{T}_n)-\min\limits_{\theta}L_{\theta}\} \nonumber \\
=&\inf\limits_{\Xi_n} \sup\limits_{\mathcal{T}_n} \mathop{\mathbb{E}}\limits_{\mathcal{S}_n} \{\sum\limits_{t} |p_t-X^T_t| -\min\limits_{\theta}\sum\limits_{t}|f^{\theta}_t-X^T_t|\}\\
\overset{(a)}{\geq} &\inf\limits_{\Xi_n} \mathop{\mathbb{E}}\limits_{\mathcal{T}_n} \mathop{\mathbb{E}}\limits_{\mathcal{S}_n} \{\sum\limits_{t} |p_t-X^T_t| -\min\limits_{\theta}\sum\limits_{t}|f^{\theta}_t-X^T_t|\}\\
=&\inf\limits_{\Xi_n} \mathop{\mathbb{E}}\limits_{\mathcal{T}_n} \mathop{\mathbb{E}}\limits_{\mathcal{S}_n} \sum\limits_{t} |p_t-X^T_t| -\mathop{\mathbb{E}}\limits_{\mathcal{T}_n} \min\limits_{\theta}\sum\limits_{t}|f^{\theta}_t-X^T_t|, \label{eq:Lower_half}
\end{align} 
where inequality (a) holds since the worst-case target sequence will generate no lower regret than compared with any other stochastic target sequences. We now assume that the target instances and the SSI instances are generated according to a joint distribution $P(X^S_t,X^T_t) = P(X^T_t)P(X^S_t|X^T_t)$, here $P(X^T_t)$ is a  Bernoulli distribution with probability $\frac{1}{2}$, i.e., $X^T_t\sim \textup{Ber}(\frac{1}{2})$. Clearly, in this case the expected loss incurred by the expert cannot be smaller than $n/2$. Then the sequential prediction problem becomes $n$ repetitive one-instance prediction problem as follows:
\begin{align}
    \inf\limits_{\tilde{\mathcal{T}}_n} \mathop{\mathbb{E}}\limits_{\mathcal{T}_n,\mathcal{S}_n} \sum\limits_{t} |\tilde{X}_t(X^S_t)-X^T_t|= n \mathop{\mathbb{E}}\limits_{X^T_1,X^S_1} |\tilde{X}_1^*(X^S_1)-X^T_1|. \label{eq:ngame}
\end{align}
It is known that for the absolute loss, the optimal forecaster $\tilde{X}_1^*$ is determined by minimising $\sum_{X^T_1} |\tilde{X}_1(X^S_1)-X^T_1| P(X^T_1|X^S_1)$ where $P(X^T_1|X^S_1)$ is induced by the joint distribution $P(X^T_1, X^S_1)$. Then we denote by $\xi^*$ the expected loss induced by $\tilde{X}_1^*$ in (\ref{eq:ngame}), and note that the optimality of $\tilde{X}_1^*$ is w.r.t. the individual loss $|\tilde{X}_1^*(X^S_1)-X^T_1|$. Following  (\ref{eq:Lower_half}), we have, 
\begin{align}
&\inf\limits_{\Xi_n} \sup\limits_{\mathcal{T}_n} \mathop{\mathbb{E}}\limits_{\mathcal{S}_n} \{L(p,\mathcal{T}_n)- \min\limits_{\theta}L_{\theta}\} \nonumber \\
\geq& \left(n\xi^* - \frac{1}{2}n\right) +\left(\frac{1}{2}n-\mathop{\mathbb{E}}\limits_{\mathcal{T}_n} \min\limits_{\theta}\sum\limits_{t}|f^{\theta}_t-X^T_t|\right)\\
\geq& \left(\xi^* - \frac{1}{2}\right)n + \sqrt{\frac{n}{2}\log N+1},
\end{align}    
where the last step is derived with the same procedures from Theorem 3.7 in \cite{cesa2006prediction}.
\end{proof}

It can be easily checked that the first term in (\ref{eq:lower}) is always negative since $\xi$ is always smaller than $\frac{1}{2}$. So for large $n$, the lower bound is negative, showing that the loss produced by our forecaster could potentially be much smaller than the best expert. It can also be seen that if the term $C_s(n)-L^*(n)$ in the upper bound takes the form $-cn$ for some positive $c$, then the upper and lower bound are matched in terms of the scaling law. In the next section, we show two examples demonstrating this point.

\section{Examples} \label{sec:example}
In this section, we consider two concrete online learning problems and derive their corresponding upper and lower bounds to verify the effectiveness of the proposed bounds. To characterize the behavior of the expert class, we will further consider the expert class generating a cumulative loss that scales linearly in $n$. The following expert class with an example displays one of the possible case satisfying the linear loss expert.

\begin{definition}(Constant Expert)
We say an expert class is the constant expert class such that all experts in the class yield a fixed prediction for any target instances. Mathematically,
\begin{align}
 \text{For all $t$ from $1$ to $n$, }f^{\theta}_t = c_{\theta},
\end{align}
where $c_{\theta}$ is some constant in $\mathcal{D}$. 
\end{definition}

Since the constant expert class is independent to the target instances, we can directly calculate the amount $L^*(n)$ defined in Theorem \ref{th:upper} under a certain setup for the decision space $\mathcal{D}$ and output space $\mathcal{X}$. We give two examples as follows.

\begin{example}
Assume the decision space is $\mathcal{D}=[0,1]$, and we consider a constant expert class that each expert predict a fixed constant $f^{\theta}_c$ in $\mathcal{D}]$. Also, assume that there always exists two experts predicting $0.1$ and $0.7$ for any time $t$. We also consider the binary output space, e.g, $\mathcal{X}=\{0,1\}$. Then we have
\begin{align}
 L^*(n)=\inf\limits_{\mathcal{T}_n}\min_{\theta}L(\mathcal{F}^{\theta},\mathcal{T}_n) = 0.1n.
\end{align}
\end{example}

\subsection{SSI via a binary symmetric channel}
In this example, we consider a learning problem setup that $\mathcal{X}=\{0,1\}, \mathcal{D}=[0,1]$ under the absolute loss $\ell(x, y)=|x-y|$. We assume that the SSI is the output of a binary symmetric channel with the flipping probability $\delta$ with the target sequence being the input. That is,
\begin{align}
    P(X^{T}_{t}=X^S_t|X^S_t) = 1-\delta,\\
    P(X^{T}_{t}=1-X^S_t|X^S_t) = \delta.
\end{align}
It can be shown that the ML estimator $\hat{\mathcal{T}}_n(\mathcal{S}_n)=\mathcal{S}_n$ when $\delta<\frac{1}{2}$, and $\hat{\mathcal{T}}_n(\mathcal{S}_n)=\bar{S}_n$ when $\delta>\frac{1}{2}$, where $\bar{S}_n:=(1-X^S_1,1-X^S_2,\dots,1-X^S_n)$ denotes the sequence consisting of the flipped SSI instances. For $\delta=\frac{1}{2}$, $\hat{\mathcal{T}}_n(\mathcal{S}_n)$ can be any value in $\mathcal{D}^{n}$. With the maximum likelihood estimator $\hat{\mathcal{T}}_n(\mathcal{S}_n)$, we can calculate the expected loss as:
\begin{align} \label{eq:csn}
\mathop{\mathbb{E}}_{\mathcal{S}_n}[L(\hat{\mathcal{T}}_n(\mathcal{S}_n),\mathcal{T}_n)]=n\left(\delta\wedge \bar{\delta}\right)
\end{align}
where $a\wedge b=\min{\{a,b\}}$ and $\bar{\delta}=1-\delta$, which satisfies Assumption \ref{assump:D} with $C_{S}(n)=n\left(\delta\wedge \bar{\delta}\right)$.

\begin{corollary} \label{coro:1}
Under the binary flipping channel setup, when $L^*(n)$ grows linearly in $n$, i.e. $L^*(n)=c_fn$, we have
\begin{align}
R(n) \leq \sqrt{\frac{n}{2}\log{(N+1)}} + \min{\left\{0,n(\delta\wedge \bar{\delta} -c_f)\right\}}.
\end{align}
\end{corollary} 

\begin{proof}
The proof is straightforwardly following Theorem \ref{th:upper} from equation (\ref{eq40})
\begin{align}
R(n) =&\sqrt{\frac{n}{2}\log(N+1)} \nonumber\\
&+\sup\limits_{\mathcal{T}_n} \{\min\{\mathop{\mathbb{E}}\limits_{\mathcal{S}_n}(L(\hat{\mathcal{T}}_n(\mathcal{S}_n),\mathcal{T}_n)-\min\limits_{\theta}L_{\theta}),0\}\}\\\\
=&\sqrt{\frac{n}{2}\log{(N+1)}} + \min\{0,n\left(\delta\wedge \bar{\delta}-c_f\right)\}.
\end{align}
\end{proof}

Notice that if any expert in the expert class $\mathcal{F}^{\theta}_n=(f_1^{\theta},f_2^{\theta},\dots,f_n^{\theta})$ suffers a cumulative loss more than $\frac{1}{2}n$, one can construct a new expert class $(1-f_1^{\theta},1-f_2^{\theta},\dots,1-f_n^{\theta})$ that suffers a loss smaller than $\frac{1}{2}n$. Hence we only consider the case that $c_f$ is always smaller or equal to $\frac{1}{2}$. From Corollary \ref{coro:1}, we notice that if $\delta \wedge \bar{\delta}$ is smaller than $c_f$, the regret is asymptotically negative and scale linearly with $n$. In the following, we give the corresponding lower bound. 
\begin{corollary}
Under the binary symmetric channel setup, we have
\begin{align} 
R(n) & \geq \sqrt{\frac{n}{2}\log{(N+1)}}-n\left(\frac{1}{2}-\delta\wedge \bar{\delta}\right).
\end{align}
\end{corollary}
\begin{proof}
Following the proof of Theorem \ref{th:lower}, we need to specify the quantity $\xi^*$. We start by finding the optimal forecaster $p_t$ for predicting the target instances $X^T_t$ by minimising the absolute loss:
\begin{align}
    &\mathop{\mathbb{E}}\limits_{X^T_t} l(p_t(X^S_t),X^T_t) =\sum_{X^T_t}\sum_{X^S_t} |p_t(X^S_t)-X^T_t| P(X^S_t,X^T_t)  \nonumber \\
    &=\sum_{X^S_t} \left(\sum_{X^T_t} |p_t(X^S_t)-X^T_t| P(X^T_t|X^S_t)\right) P(X^S_t) .
\end{align}
Then minimising the expected loss w.r.t. $p_t(X^S_t)$ is equivalently minimising $\sum_{X^T_t} |p_t(X^S_t)-X^T_t| P(X^T_t|X^S_t)$ for any $X^S_t$. Given $X^S_t=1$, we have
\begin{align}
    &\sum_{X^T_t} |p_t(X^S_t=1)-X^T_t| P(X^T_t|X^S_t=1)\\
    &= (1-\delta)|p_t(X^S_t=1)-1| + \delta|p_t(X^S_t=1)-0|. \label{eq39}
\end{align}
Then we can obtain when $\delta<\frac{1}{2}$, $p^*_t(X^S_t)=X^S_t$, when $\delta>\frac{1}{2}$, $p^*_t(X^S_t)=X^S_t$, and when $\delta=\frac{1}{2}$, there are an infinite number of minimizers $p^*_t(X^S_t)$ between $0$ and $1$. One can verify that the optimal forecaster $p^*_t(X^S_t)$ is the maximum likelihood estimator $\hat{\mathcal{T}}_n(\mathcal{S}_n)$. Then similar to (\ref{eq:csn}), we have
\begin{align}
    \xi^*&=\mathop{\mathbb{E}}\limits_{X^T_t}\mathop{\mathbb{E}}\limits_{X^S_t}|p^*_t-X^T_t|=\frac{1}{n}\mathop{\mathbb{E}}_{\mathcal{S}_n}[L(\hat{\mathcal{T}}_n(\mathcal{S}_n),\mathcal{T}_n)]\\
    &=\delta\wedge \bar{\delta}. \label{eq:xi1}
\end{align}
By substituting the $\xi^*$ in Theorem \ref{th:lower} as (\ref{eq:xi1}), we completed the proof.
\end{proof}
We see that in the case when $\delta \wedge \bar \delta < c_f$, the upper and the lower bound is matched in terms of the scaling law of order $\Omega(-n)$ (although with a different constant). 

\subsection{SSI via a Zero-mean Gaussian Channel} 
Now we consider a different type of side information such that the side instance is the noisy version of the target instance pair-wise: $X_{S,t}=X_{T,t}+N_t$, where $N_t\sim \mathcal{N}(0,\sigma^2)$. Here we still assume the target instances are restricted in a binary outcome space $\{0,1\}$. 
Note that in this problem setup, the side instances are drawn from a distribution over the space $R$, which differs from the instance space $\mathcal{X}$. 

We can easily determine the maximum likelihood estimator $\hat{X}^T_t(X^S_t)$ for this problem: $\hat{X}^T_t(X^S_t)=1$ when $X^S_t\geq\frac{1}{2}$, and $\hat{X}^T_t(X^S_t)=0$ when $X^S_t<\frac{1}{2}$. By introducing the cumulative density function of the standard normal distribution $\Phi(z)=\frac{1}{\sqrt{2\pi}}\int^{z}_{-\infty}e^{-t^2/2}dt$, we have
\begin{align} \label{eq:csn2}
\mathop{\mathbb{E}}_{\mathcal{S}_n}[L(\hat{\mathcal{T}}_n(\mathcal{S}_n),\mathcal{T}_n)]=n\Phi(-\frac{1}{2\sigma}).
\end{align}

\begin{corollary} \label{coro:zero-mean}
Under the zero-mean Gaussian channel setup, and when $L^*(n)$ is linear to $n$, i.e. $L^*(n)=c_fn$, we have
\begin{align}
R(n) \leq \sqrt{\frac{n}{2}\log{(N+1)}} + \min\{0,n\left(\Phi(-\frac{1}{2\sigma})-c_f\right)\}
\end{align}
\end{corollary} 

\begin{proof}
The proof of Corollary~\ref{coro:zero-mean} follows a similar procedure as in the proof of Corollary~\ref{coro:1}, but differs in calculating $\mathop{\mathbb{E}}_{\mathcal{S}_n}[L(\hat{\mathcal{T}}_n(\mathcal{S}_n),\mathcal{T}_n)]$, i.e. $\mathop{\mathbb{E}}_{\mathcal{S}_n}[L(\hat{\mathcal{T}}_n(\mathcal{S}_n),\mathcal{T}_n)]=n\Phi(-\frac{1}{2\sigma})$. Similar to the equation (\ref{eq:csn}), the assumption 1 holds.
\end{proof}

It can be seen that the upper bound in this example behaves similarly to that in the binary symmetric channel case. When the quantity $\Phi(-\frac{1}{2\sigma})$ is smaller than $c_f$, the upper bound of the minimax regret becomes negative with a large $n$. Intuitively, when $\sigma$ is large, the quantity $\Phi(-\frac{1}{2\sigma})$ will become larger, which decreases the effectiveness of the SSI. 

\begin{corollary}
Under the zero-mean Gaussian channel setup, we have
\begin{align} 
R(n) & \geq \sqrt{\frac{n}{2}\log{(N+1)}}+n\left(\Phi(-\frac{1}{2\sigma})-\frac{1}{2}\right).
\end{align}

\end{corollary}
\begin{proof}
Similar to the proof of the Corollary 2, by considering $X^S_t=1$, we have
\begin{align}
   &\sum_{X^T_t} |p_t(X^S_t=1)-X^T_t| P(X^T_t|X^S_t=1)\label{eq45}\\
    &= \frac{e^\frac{(X^S_t - 1)^2}{2\sigma^2}}{e^\frac{(X^S_t - 1)^2}{2\sigma^2}+e^\frac{(X^S_t - 0)^2}{2\sigma^2}}|p_t(X^S_t=1)-1| \\
    &+ \frac{e^\frac{(X^S_t - 0)^2}{2\sigma^2}}{e^\frac{(X^S_t - 1)^2}{2\sigma^2}+e^\frac{(X^S_t - 0)^2}{2\sigma^2}}|p_t(X^S_t=1)-0| \label{eq:example2pt}
\end{align}
It can be verified that the optimal solution minimising (\ref{eq45}) is the maximum likelihood estimator $\hat{\mathcal{T}}_n(\mathcal{S}_n)$. Then we have
\begin{align}
    \xi^*&=\frac{1}{n}\mathop{\mathbb{E}}_{\mathcal{S}_n}[L(\hat{\mathcal{T}}_n(\mathcal{S}_n),\mathcal{T}_n)]=\Phi(-\frac{1}{2\sigma})\label{eq:expxi}
\end{align}
Then by substituting $\xi^*$ in Theorem \ref{th:lower} with (\ref{eq:expxi}), we completed the proof.
\end{proof}
Similarly, as $\sigma>0$, we have $\Phi(-\frac{1}{2\sigma})<\Phi(0)=\frac{1}{2}$, the lower bound will become negative when $n$ increases. Similar to the binary symmetric channel example, the upper and the lower bound is matched in terms of the scaling law if $c_f > \Phi(-\frac{1}{2\sigma})$. 

\section{Conclusion and Future Works}
This work shows the upper and lower regret bounds on general deterministic online learning problems with two concrete examples, where an additional stochastic sequential side information sequence is revealed to the forecaster. The result infers the effectiveness of the the side information which may significantly improved the learning rate and shows the possibility of producing a negative regret. For future works, one may wish to find a tighter lower bound on the minimax regret based on more advanced algorithms, or more elementary proofs. 

\bibliographystyle{IEEEtranN}
\bibliography{conference_101719}

\end{document}